\documentclass[sigconf,nonacm]{acmart}

\AtBeginDocument{%
  }

\setcopyright{acmlicensed}
\copyrightyear{2027}
\acmYear{2027}
\acmDOI{XXXXXXX.XXXXXXX}
\acmConference[SAC'27]{The 42nd ACM/SIGAPP Symposium on Applied Computing}{April 5--9, 2027}{Gwangju, South Korea}
\acmISBN{979-X-XXXX-XXXX-X/27/04}

\usepackage{xspace}
\usepackage{multirow}
\usepackage{subcaption}

\newcommand{\sisa}{SISA\xspace}             %
\newcommand{\sisarope}{SISA w/ RoPE\xspace} %
\newcommand{\fox}{FoX\xspace}               %

\newif\ifhidepending
\hidependingfalse

\newcommand{\Ct}{\tilde{C}}
\newcommand{\Bt}{\tilde{B}}
\newcommand{\vect}[1]{\mathbf{#1}}

\usepackage{tikz}
\usetikzlibrary{positioning,arrows.meta,calc}

\providecommand{\se}[1]{\,{\scriptsize$\pm$#1}}
\providecommand{\single}{\textsuperscript{\dag}}
\providecommand{\best}[1]{\textbf{#1}}

\begin{document}

\title{Forget Attention: Importance-Aware Attention Is All You Need}

\author{Suhyeong Shin}
\affiliation{%
  \department{Department of Computer Engineering}
  \institution{Kangwon National University}
  \country{Republic of Korea}}
\email{shin0815@kangwon.ac.kr}

\author{Yeongwook Yang}
\authornote{Corresponding author.}
\affiliation{%
  \department{Department of Computer Engineering}
  \institution{Kangwon National University}
  \country{Republic of Korea}}
\email{yeongwook.yang@kangwon.ac.kr}
\renewcommand{\shortauthors}{Shin and Yang}

\begin{abstract}
Attention retrieves content precisely, but its scores depend only on query--key products, leaving order and recency to positional encodings. Conversely, state-space models (SSMs) such as Mamba-3 track order and recency through input-dependent state transitions that decay and rotate, yet compressing context into a fixed-size state hampers precise retrieval. Hybrids like Samba and Hymba exploit this complementarity only across layers or heads, so SSM transitions never shape attention scores directly. The closest prior work, the Forgetting Transformer (FoX), adds an input-dependent forget gate to the scores but reweights attention by recency alone. We propose SSM-Informed Softmax Attention (\sisa), which builds Mamba-3-style state transitions, rotation included, directly into the scores alongside FoX's gate. Each logit receives the weight such an SSM would give that token pair: a decayed match between rotated SSM projections. This rotation already encodes relative position, so \sisa drops RoPE; restoring it raises loss. \sisa outperforms parameter-matched Transformer and FoX in language modeling and length extrapolation. Among the attention models we compare, hybrids included, only \sisa benefits from context beyond the training window. Building SSM transitions into the scores yields stronger attention, and our kernels make it faster and leaner than baselines on their official or standard kernels.
\end{abstract}



\maketitle

\section{Introduction}\label{sec:intro}

\begin{figure*}[t]
  \centering
  \includegraphics[width=\textwidth]{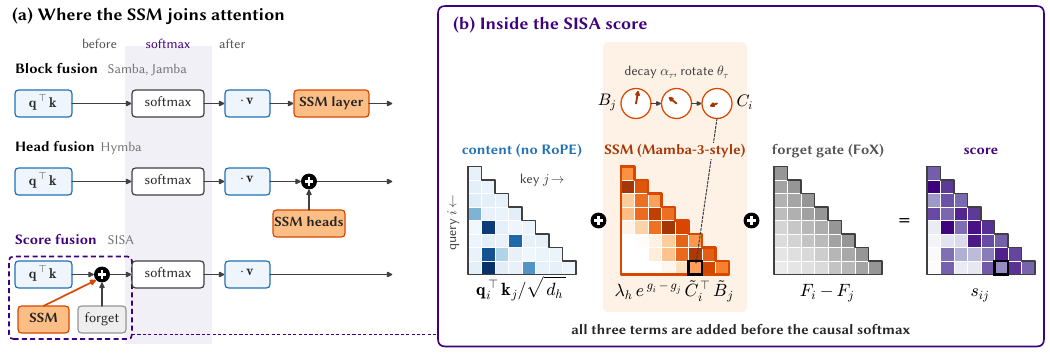}
  \Description{Two panels. Panel (a), "Where the SSM joins attention", shows three rows that share the attention pipeline q-transpose-k, softmax, times v, with the softmax boxes aligned in one shaded column labeled softmax, "before" to its left and "after" to its right. Block fusion (Samba, Jamba): an SSM layer follows the attention output. Head fusion (Hymba): SSM heads are added to the attention heads' outputs after the softmax. Score fusion (SISA): an SSM term and a forget term are added to q-transpose-k before the softmax; this part is outlined and connected to panel (b). Panel (b), "Inside the SISA score", shows four lower-triangular 8-by-8 maps of query i by key j: the content map q_i transposed k_j over the square root of d_h (no RoPE) with scattered strong cells, plus the SSM map lambda_h times exp(g_i minus g_j) times C-tilde_i transposed B-tilde_j with a narrow band next to the diagonal, plus the forget-gate map F_i minus F_j that fades with distance, equals the score map s_ij. An inset above the SSM map shows a state written by B_j passing through three circles, rotating and shrinking at each step, and read by C_i; a dashed arrow links it to one highlighted cell, which is also highlighted in the score map. A footer reads: all three terms are added before the causal softmax.}
  \caption{Where the SSM joins attention. (a) Hybrids combine SSM and attention outside the scores, across layers or heads; \sisa adds the SSM to the scores. (b) The \sisa score as a sum of three causal maps (Eq.~\eqref{eq:score}; maps schematic).}
  \label{fig:schematic}
\end{figure*}

Attention~\cite{vaswani2017attention} retrieves content precisely, but its scores depend only on query--key products, leaving order and recency, beyond the causal mask, to a positional encoding such as rotary position embedding (RoPE)~\cite{su2024roformer}, which ties attention to the distances seen in training~\cite{press2022alibi}.
State-space models (SSMs) such as Mamba-3~\cite{lahoti2026mamba3} are the converse: they track order and recency through input-dependent state transitions that decay and rotate, but compressing context into a fixed-size state hampers precise retrieval (Table~\ref{tab:main}).

Hybrids exploit this complementarity only across layers (Jamba~\cite{lieber2024jamba}, Samba~\cite{ren2025samba}) or heads (Hymba~\cite{dong2025hymba}) (Fig.~\ref{fig:schematic}a): their attention logits still come from $\vect{q}^{\top}\vect{k}$ alone, so SSM transitions never shape attention scores directly.
The closest prior work, the Forgetting Transformer (\fox)~\cite{lin2025fox}, adds an input-dependent forget gate to the scores but reweights attention by recency alone, which in our runs does not replace RoPE in held-out loss (Section~\ref{sec:ablation}).

We propose \emph{SSM-Informed Softmax Attention} (\sisa), which builds Mamba-3-style state transitions, rotation included, directly into the attention score, alongside \fox's forget gate (Fig.~\ref{fig:schematic}b).
Each logit receives the weight such an SSM would give that token pair: a decayed match between SSM projections rotated by a data-dependent angle.
Softmax still selects keys by content, now informed by the SSM's order and recency.
Because the rotation already encodes relative position, in content-dependent units, \sisa drops RoPE on $\vect{q}^{\top}\vect{k}$; adding it back raises the loss and removes the gain beyond the training window (Section~\ref{sec:ablation}).
Our contributions are:
\begin{itemize}
\item \textbf{Method.} The \sisa score and layer, designed at 50M and transferred unchanged to 152M and 369M, with Triton~\cite{tillet2019triton} kernels validated against an fp64 reference (Section~\ref{sec:method}).
\item \textbf{Quality.} Trained on 10B FineWeb-Edu tokens~\cite{penedo2024fineweb}, \sisa has lower held-out cross-entropy (CE) than the parameter-matched \fox and Transformer, by $0.022$ and $0.036$ nats at 50M, $0.014$ and $0.022$ at 152M, and $0.004$ and $0.015$ at 369M (3 seeds each, paired SE $\le$\,0.002); it also has the best passkey retrieval among attention-only models at 50M and, narrowly, at 152M (Section~\ref{sec:results}).
\item \textbf{Length generalization.} \sisa is the only model with attention that we compare, hybrids included, whose loss falls with context beyond its 2K training window ($-0.028\pm0.002$ nats on FineWeb-Edu, $-0.050\pm0.005$ on the Pile at 50M). At 50M its CE does not degrade up to 64K tokens, and at 152M it retrieves 86\% of passkeys at twice the training length, vs.\ at most 37\% for every other model measured (Section~\ref{sec:longctx}).
\item \textbf{Efficiency.} With our kernels, \sisa trains faster and uses less memory than every baseline measured on its official or standard kernels, at a 2K context and the same micro-batch on each of four GPUs (H100, RTX A6000, RTX 4090 and RTX PRO 5000 Blackwell), e.g.\ 1.3--1.8$\times$ the Transformer's throughput with 1.9--4.2$\times$ less peak memory (Section~\ref{sec:efficiency}).
\item \textbf{Ablations.} Removing the rotation or restoring RoPE costs $0.013$--$0.014$ nats at 50M and erases the beyond-window gain; with RoPE present (base recipe), removing the rotation has little effect on CE (Section~\ref{sec:ablation}).
\end{itemize}
\sisa changes only the attention scores and adds no recurrent layers.
We compare it both with attention-only models and with models that use recurrent layers: Samba and Mamba-3 reach lower loss at larger scale, and Hymba at equal tokens, but within Hymba's GPU-time budget \sisa reaches lower CE (50M; Section~\ref{sec:results}).

\section{Related Work}\label{sec:related}
\paragraph{State-space models.}
Mamba~\cite{gu2024mamba} makes SSM parameters input-dependent; Mamba-2~\cite{dao2024mamba2} shows that a scalar-decay SSM is masked linear attention with weights $e^{g_i-g_j}C_i^{\top}B_j$; Mamba-3~\cite{lahoti2026mamba3} adds a complex-valued state update, computed as a data-dependent rotary embedding of $B$ and $C$.
\sisa adds this rotated, decayed kernel to the softmax logit (Eq.~\eqref{eq:score}) rather than weighting values with it.

\paragraph{Hybrid models.}
Jamba~\cite{lieber2024jamba} (without positional encoding) and Samba~\cite{ren2025samba} interleave attention and Mamba layers; Hymba~\cite{dong2025hymba} runs attention and SSM heads in parallel.
Their attention logits come from $\vect{q}^{\top}\vect{k}$ alone, so SSM transitions reach them at most through the inputs of $\vect{q}$ and $\vect{k}$; \sisa places the transitions in the logits and adds no recurrent layer.

\paragraph{Gated and data-dependent attention.}
\fox~\cite{lin2025fox} adds the cumulative log forget gate $F_i-F_j$ to the logits, a data-dependent ALiBi; as $F_i-F_j\le 0$, it can only down-weight distant keys.
Selective RoPE~\cite{movahedi2026selrope} rotates $\vect{q}$ and $\vect{k}$ by accumulated input-dependent angles and pairs this with the model's gates, \fox's included; PaTH~\cite{yang2025path} replaces RoPE's rotation between $\vect{q}_i$ and $\vect{k}_j$ with a product of data-dependent Householder-like transforms, also combined with \fox.
Both encode position in $\vect{q}^{\top}\vect{k}$ itself.
PaTH's authors also note that the direct softmax analogue of Mamba-2, which multiplies $\vect{q}_i^{\top}\vect{k}_j$ by the cumulative decay inside the exponential, greatly underperformed \fox in their preliminary experiments~\cite{yang2025path}.
\sisa instead keeps \fox's term and an undecayed $\vect{q}^{\top}\vect{k}$ without positional encoding, and places a data-dependent rotation (relative phase $\Phi_i-\Phi_j$) and decay in a separate SSM term on its own projections $C$ and $B$, whose signed, content-dependent contribution is added to the logit and decays with distance.

\paragraph{Positional encoding.}
RoPE~\cite{su2024roformer} makes $\vect{q}_i^{\top}\vect{k}_j$ depend on $i-j$ but, unmodified, degrades beyond the training length~\cite{press2022alibi,chen2023pi}; ALiBi~\cite{press2022alibi} adds a fixed head-specific penalty proportional to $i-j$ and extrapolates further.
CoPE~\cite{golovneva2024cope} conditions position on context, incrementing it only at tokens the model selects.
Without positional encoding (NoPE), causal Transformers infer position from the mask~\cite{haviv2022nope} and can length-generalize better~\cite{kazemnejad2023nope}.
\sisa drops RoPE as a consequence: its rotation already encodes relative position (Section~\ref{sec:ablation}).

\paragraph{Attention kernels.}
FlashAttention~\cite{dao2022flashattention,dao2024flashattention2} computes exact attention in tiles without materializing the $L\times L$ scores; Triton~\cite{tillet2019triton} makes such kernels practical for modified scores, as in the official \fox kernel and ours (Section~\ref{sec:method:kernels}).

\section{SSM-Informed Softmax Attention}\label{sec:method}

\sisa replaces the attention sublayer of a pre-norm Transformer block (RMSNorm~\cite{zhang2019rmsnorm}, SwiGLU MLP~\cite{shazeer2020glu}, tied embeddings).
Each head keeps one causal softmax, but its logits add the weight a Mamba-3-style SSM would give each token pair and, second, \fox's forget gate; $\vect{q}^{\top}\vect{k}$ carries no positional encoding.

\subsection{Attention score}\label{sec:method:score}
Let $d$ be the model width, $H$ the number of heads, $L$ the sequence length and $\sigma$ the logistic sigmoid.
For one head with head dimension $d_h$ and state dimension $d_s$, linear maps of the normalized input $\bar{x}_t\in\mathbb{R}^{d}$ give $\vect{q}_t,\vect{k}_t,\vect{v}_t\in\mathbb{R}^{d_h}$, SSM input and output projections $B_t,C_t\in\mathbb{R}^{d_s}$, rotation angles $\theta_t\in\mathbb{R}^{d_s/2}$, and two scalars $a_t$ (decay) and $z_t$ (forget gate).
Their cumulative sums
\begin{equation}\label{eq:cum}
\begin{gathered}
  \Phi_t=\textstyle\sum_{\tau\le t}\theta_\tau,\qquad
  \Ct_t=R(\Phi_t)\,C_t,\qquad \Bt_t=R(\Phi_t)\,B_t,\\
  g_t=-\textstyle\sum_{\tau\le t}\operatorname{softplus}(a_\tau),\qquad
  F_t=\textstyle\sum_{\tau\le t}\log\sigma(z_\tau)
\end{gathered}
\end{equation}
define the rotated projections ($R(\Phi)$ is block-diagonal with $d_s/2$ planar rotations), the log-decay $g$ and the log-forget $F$.
The score between query $i$ and key $j\le i$ is
\begin{equation}\label{eq:score}
  s_{ij} \;=\; \frac{\vect{q}_i^{\top}\vect{k}_j}{\sqrt{d_h}}
  \;+\; \lambda_h\, e^{\,g_i - g_j}\, \Ct_i^{\top}\Bt_j
  \;+\; \bigl(F_i - F_j\bigr),
\end{equation}
and the output is $\vect{o}_i=\sum_{j\le i}\operatorname{softmax}_j(s_{ij})\,\vect{v}_j$.
Unrolled,
\begin{equation}\label{eq:unroll}
\begin{gathered}
  e^{\,g_i-g_j}=\textstyle\prod_{\tau=j+1}^{i}\alpha_\tau,\qquad
  F_i-F_j=\textstyle\sum_{\tau=j+1}^{i}\log f_\tau,\\
  \Ct_i^{\top}\Bt_j=C_i^{\top}R(\Phi_i-\Phi_j)^{\top}B_j,
\end{gathered}
\end{equation}
with per-token decay $\alpha_\tau=e^{-\operatorname{softplus}(a_\tau)}\in(0,1)$ and forget gate $f_\tau=\sigma(z_\tau)\in(0,1)$.

The \emph{SSM term}, the core of \sisa, is the weight with which a selective SSM whose state decays by $\alpha_\tau$ and rotates by $\theta_\tau$ passes token $j$ to position $i$: the scalar-decay mixing weight of Mamba-2's structured state-space duality (SSD)~\cite{dao2024mamba2}, with $C$ and $B$ rotated as in Mamba-3~\cite{lahoti2026mamba3}.
Unlike in both, the step $\operatorname{softplus}(a_\tau)$ that sets the decay does not also scale the input $B_\tau$ or (as in Mamba-3) the angle $\theta_\tau$, and we omit Mamba-3's exponential-trapezoidal discretization and multi-input multi-output (MIMO) variant.
Fig.~\ref{fig:schematic}b draws Eq.~\eqref{eq:score} as a sum of three causal maps; its inset unrolls one SSM entry with $i-j=2$, a state written by $B_j$ that decays and rotates at each step until $C_i$ reads it. \sisa computes this weight in closed form and never forms the state.
An SSM mixes values with this weight, an input-dependent importance of key $j$ for query $i$; \sisa adds it to the logit, where softmax combines it with the query--key match; $\lambda_h>0$ is a learned per-head scale (a softplus).
The \emph{forget term}, the second component, is the bias of \fox~\cite{lin2025fox}: data-dependent, non-positive and non-increasing as key $j$ moves away from query $i$.
The \emph{content term} is the usual attention logit, without positional encoding.

\subsection{Relative position from the rotation}\label{sec:method:nope}
Within the SSM term, the rotation enters only through the relative phase $\Phi_i-\Phi_j$: with constant angles $\theta_t\equiv\omega$, $\Ct_i^{\top}\Bt_j$ would be RoPE on $(C,B)$ instead of $(\vect{q},\vect{k})$, and with $\theta_t$ computed from $\bar{x}_t$, each token advances the rotation by its own amount, measuring relative position in content-dependent units.
Recency comes from $e^{g_i-g_j}\le1$ and $F_i-F_j\le 0$, both decreasing monotonically with distance, and since every signal is a difference of cumulative sums, none is tied to absolute position.
\sisa therefore drops RoPE on $\vect{q}^{\top}\vect{k}$: in a controlled 50M ablation, restoring it raises held-out CE by $0.0126\pm0.0017$ and turns the beyond-window gain into a loss, while with RoPE present (base recipe), removing the rotation has little effect on CE (Section~\ref{sec:ablation}).

\subsection{Parameters and initialization}\label{sec:method:params}
Per layer, the SSM term adds $W_B,W_C\in\mathbb{R}^{Hd_s\times d}$, $W_\theta\in\mathbb{R}^{(Hd_s/2)\times d}$ and, per head, a decay projection $(\vect{w}_a,b_a)$ and $\lambda_h$; the forget gate adds $(\vect{w}_z,b_z)$ per head.
With $d_h=64$ and $d_s=32$ ($Hd_s=d/2$) this is $\tfrac{5}{4}d^2+dH+2H\approx1.27\,d^2$ parameters per layer, plus $dH+H$ for the forget gate; \sisa offsets them with a narrower SwiGLU hidden width (Section~\ref{sec:setup}), keeping the total within 0.007\% of the Transformer at every scale.
Weights are drawn from $\mathcal{N}(0,0.02^2)$, except $W_\theta\sim\mathcal{N}(0,0.01^2)$, $b_a=-5$ (initial decay half-life ${\approx}100$ tokens), $b_z=4$ ($f\approx0.98$) and $\lambda_h=\operatorname{softplus}(-1)$.
We set $W_C=0$, so that \sisa starts as \fox (NoPE), and train $\lambda_h$ without weight decay at the main learning rate, together lowering CE by $0.0038\pm0.0010$ at 50M (Section~\ref{sec:ablation}).

\subsection{Complexity and exact chunked evaluation}\label{sec:method:chunk}
Like softmax attention, \sisa costs $O(L^2)$ time and, with tiling, $O(L)$ memory; each query--key pair adds a $d_s$-dimensional dot product (a quarter more matrix-multiply work in the forward attention core), a scalar decay factor and $F_i-F_j$.
We evaluate Eq.~\eqref{eq:score} tile by tile with an online softmax, as in FlashAttention~\cite{dao2022flashattention,dao2024flashattention2}, and handle the pairwise decay as in the chunked SSD algorithm~\cite{dao2024mamba2}: with 128-token chunks, a diagonal tile evaluates $e^{g_i-g_j}$ directly, and an off-diagonal tile with query chunk $p$ and earlier key chunk $c$ uses
\begin{equation}\label{eq:chunk}
  e^{\,g_i-g_j}=e^{\,g_i-r_p}\cdot e^{\,r_p-r'_c}\cdot e^{\,r'_c-g_j},
\end{equation}
where $r_p$ is the first value of $g$ in chunk $p$ and $r'_c$ the last in chunk $c$.
All three factors are at most 1, so nothing overflows and no clamping is needed; the outer factors are folded into $\Ct$ and $\Bt$ and the middle one is a scalar per tile.
For decoding, \sisa caches $\Bt_j$, $g_j$ and $F_j$ next to $\vect{k}_j,\vect{v}_j$ ($d_s+2$ extra values per head and token).

\subsection{Training kernels}\label{sec:method:kernels}
We implement \sisa in Triton~\cite{tillet2019triton} in two versions.
The \emph{reference path} runs the chunked kernel above, the rest of the attention sublayer in eager PyTorch, the MLP, RMSNorms, LM head and loss under torch.compile, and a fused AdamW; every \sisa model in this paper was trained with it except the 369M \sisarope, which used an early version of our kernels limited to item~(i) below.
\emph{Our kernels}, used for the measurements in Section~\ref{sec:efficiency}, additionally
(i)~compute seven of a layer's eight per-token projections with one matrix multiplication (the forget gate's stays in fp32) and fuse the softplus, log-sigmoid, cumulative sums and rotations into two kernels per direction;
(ii)~fuse each residual addition with the next RMSNorm, and SwiGLU's gate and up projections into one matrix multiplication;
(iii)~compute the tied LM head and cross-entropy chunk by chunk, never materializing the full logit tensor; and
(iv)~capture each micro-batch step as one CUDA graph; (ii)--(iv) are generic, not specific to \sisa.
Both paths use bf16 matrix inputs with fp32 accumulation and keep $g$, $F$ and $\Phi$ in fp32.
We test our kernels against a separate fp64 implementation of Eq.~\eqref{eq:score} on trained 50M (two sequence lengths) and 152M checkpoints and two initializations, with pre-registered criteria relative to the reference path's own deviations from fp64, $\Delta_{\mathrm{ref}}$ in loss and $e_{\mathrm{ref}}$ in relative gradient error (loss deviation at most $\max(2\Delta_{\mathrm{ref}},10^{-3})$ nats, per-parameter-group relative gradient error at most $1.25\,e_{\mathrm{ref}}+10^{-3}$).
All five cases pass, and on the trained checkpoints their gradient is closer to fp64 than the reference path's.
The kernel configurations tuned for the A6000, RTX 4090 and Blackwell pass the same test, eagerly and under CUDA-graph replay.

\section{Experimental Setup}\label{sec:setup}

\paragraph{Data and recipe.}
All models are trained from scratch on 10B FineWeb-Edu tokens~\cite{penedo2024fineweb} (19,073 steps of 524,288 tokens) with the GPT-NeoX tokenizer~\cite{black2022gptneox} and a 2,048-token context, using AdamW~\cite{loshchilov2019adamw} ($\beta{=}(0.9, 0.95)$, weight decay 0.1), gradient clipping at 1.0, fp32 master weights with bf16 computation, and a cosine schedule (peak learning rate $6{\times}10^{-4}$ with 500 warm-up steps at 50M and 152M; $3{\times}10^{-4}$ with 2,000 at 369M), untuned and shared by all models.

\paragraph{Scales and seeds.}
The Transformer has $d{=}512$/768/1024 with 6/12/24 layers and 8/12/16 heads at 50M/152M/369M.
\sisa keeps these shapes with a narrower SwiGLU MLP (1,829/2,744/2,507 instead of 2,048/3,072/2,944 hidden units); every model matches the Transformer's parameter count within 0.11\% (\sisa within 0.007\%).
We train three seeds per model at 50M and 152M and, at 369M, three for \sisa, \fox and the Transformer, two for Samba, and one for Mamba-3, \fox (NoPE) and \sisarope; Hymba, our slowest baseline (Section~\ref{sec:efficiency}), is trained only at 50M and 152M (compute budget).
Seeds, shared by all models, fix the data order, so differences are paired by seed where possible; we report mean $\pm$ standard error.

\paragraph{Baselines.}
\sisa, \sisarope, \fox and the Transformer are \emph{attention-only}; Samba and Hymba add recurrent layers, and Mamba-3 has no attention.
The \emph{Transformer} uses RoPE~\cite{su2024roformer}, RMSNorm, a SwiGLU MLP and PyTorch SDPA (FlashAttention~\cite{dao2024flashattention2}).
\emph{\fox}~\cite{lin2025fox} runs the official Triton kernel in the Transformer's LLaMA-style block (without the output gate, QK-norm and KV-shift of its Pro block); our main \fox keeps RoPE (``\fox + RoPE'' in the tables), which lowered its CE by $0.009$ at 50M and $0.010$ at 152M; the official default without RoPE is reported as \fox (NoPE).
\emph{Mamba-3}~\cite{lahoti2026mamba3} uses the official implementation: 6/12 layers at 50M/152M (the Transformer's depth; state size 160, expansion 5/5.08 to match parameters) and, at 369M, 24 blocks that each add a SwiGLU MLP to a Mamba-3 layer (state size 128, expansion 2).
\emph{Samba}~\cite{ren2025samba} and \emph{Hymba}~\cite{dong2025hymba} run their official model code, Samba with its official 2,048-token attention window (our context length) and, like all our models in training, an unpadded 50,277-token vocabulary (its default pads to a multiple of 64).
Hymba's released code (revision c085c638) zero-initializes the step-size bias, so every SSM step size starts at $\Delta{\approx}0.69$ (a state half-life of at most one token) and in a 500-step test still lies within 0.690--0.694; we therefore use the standard Mamba initialization of step size and convolution.
\sisarope adds RoPE on $\vect{q}_i^{\top}\vect{k}_j$ and keeps the base recipe (default $W_C$ initialization instead of $W_C{=}0$, and a fixed $\lambda_h$ learning rate of $6{\times}10^{-5}$); it serves as a reference.

\paragraph{Evaluation.}
Our main metric is cross-entropy (CE) in nats per token on 4,194,304 held-out FineWeb-Edu tokens.
LAMBADA~\cite{paperno2016lambada} is scored by greedy last-word exact match (5,153 items).
Passkey retrieval, adapted from~\cite{mohtashami2023landmark}, hides a five-digit key at 25--75\% depth in Pile~\cite{gao2020pile} text (300 prompts per seed, 2K to 16K tokens).
For length generalization, we evaluate without RoPE rescaling on the same 2,097,152 held-out tokens in 2K- to 64K-token windows.
To isolate the use of context beyond the training window, we score positions 8,192--16,383 of 64 documents (${\geq}$16,385 tokens; 524,288 targets per corpus) from held-out FineWeb-Edu and the Pile and report $\Delta_{>2K}$ (Table~\ref{tab:beyond}), the CE with the full causal context minus the CE with a sliding 2,048-token window (stride 128); as every sliding-window target sees nearly 2K tokens, a negative $\Delta_{>2K}$ is a gain from context beyond 2K, not a cold-start artifact.
Zero-shot HellaSwag~\cite{zellers2019hellaswag}, ARC-Easy~\cite{clark2018arc}, PIQA~\cite{bisk2020piqa} and WinoGrande~\cite{sakaguchi2020winogrande}, and ARC-Challenge, OpenBookQA~\cite{mihaylov2018obqa}, BoolQ~\cite{clark2019boolq} and SocialIQA~\cite{sap2019siqa} with the LM Evaluation Harness~\cite{gao2024lmeval}, complete the suite.

\paragraph{Time budget and efficiency.}
At 50M we also train \sisa on 22.4B tokens (2.24 passes over the same data, three seeds) against Hymba at 10B (Table~\ref{tab:equaltime}).
This budget follows the 2.24$\times$ throughput ratio of an unfused \sisa implementation over Hymba on an A6000; at the ratios we now measure (H100: 2.5$\times$ with the reference path, 4.8$\times$ with our kernels), \sisa's 22.4B tokens stay within Hymba's GPU time for 10B, taking 10\% (H100) and 5\% (A6000) less with the reference path that trained them, and 53\% and 38\% less with our kernels.
Throughput and peak memory are measured on one H100 PCIe, RTX A6000, RTX 4090 or RTX PRO 5000 Blackwell GPU that runs no other job (other GPUs of the same host were in use), at the same micro-batch for all models (16 sequences at 50M, 8 at 152M) and with every vocabulary padded to 50,304 (a multiple of 64); throughput is the steady-state rate over optimizer steps 21--48, timed at synchronized step boundaries after warm-up.
Baselines are timed with the code that trained them (with two configuration changes on the Blackwell) and \sisa with our kernels (Section~\ref{sec:efficiency}).

\section{Language Modeling Results}\label{sec:results}

\begin{table*}[t]
  \caption{Language modeling. CE: held-out cross-entropy (nats/token, mean\,$\pm$\,SE over seeds, Section~\ref{sec:setup}); LAMB.: LAMBADA accuracy (\%); PK: passkey accuracy at 2K.
  Bold: best per block; \single~single seed; --: not run.}
  \label{tab:main}
  \small\setlength{\tabcolsep}{3.2pt}
  \centering
  \begin{tabular}{l ccc ccc ccc}
\toprule
 & \multicolumn{3}{c}{50M} & \multicolumn{3}{c}{152M} & \multicolumn{3}{c}{369M} \\
\cmidrule(lr){2-4}\cmidrule(lr){5-7}\cmidrule(lr){8-10}
Model & CE$\downarrow$ & LAMB.$\uparrow$ & PK$\uparrow$ & CE$\downarrow$ & LAMB.$\uparrow$ & PK$\uparrow$ & CE$\downarrow$ & LAMB.$\uparrow$ & PK$\uparrow$ \\
\midrule
SISA (ours) & \best{3.2442\se{.0003}} & \best{16.9} & \best{0.63} & \best{2.9239\se{.0006}} & \best{21.4} & \best{0.92} & \best{2.8205\se{.0005}} & \best{23.1} & \best{0.93} \\
SISA w/ RoPE & 3.2595\se{.0013} & \best{16.9} & 0.14 & 2.9341\se{.0009} & 20.1 & 0.63 & 2.8289\single & 21.9\single & 0.60\single \\
FoX + RoPE & 3.2666\se{.0016} & 16.0 & 0.10 & 2.9375\se{.0002} & 20.6 & 0.70 & 2.8243\se{.0003} & 22.2 & 0.73 \\
FoX (NoPE, default) & 3.2752\se{.0005} & 15.3 & 0.37 & 2.9476\se{.0017} & 19.6 & 0.91 & 2.8311\single & 21.7\single & 0.90\single \\
Transformer & 3.2802\se{.0011} & 15.3 & 0.14 & 2.9458\se{.0009} & 19.8 & 0.63 & 2.8355\se{.0002} & 21.8 & 0.81 \\
\midrule
Mamba-3 & 3.2677\se{.0016} & 14.5 & 0.00 & 2.9269\se{.0014} & \best{20.9} & 0.06 & 2.8089\single & 22.7\single & 0.23\single \\
Samba & 3.2416\se{.0017} & 15.8 & \best{0.55} & \best{2.9073\se{.0011}} & 20.8 & \best{0.95} & \best{2.7722\se{.0002}} & \best{24.6} & \best{0.94} \\
Hymba & \best{3.2286\se{.0012}} & \best{16.3} & 0.19 & 2.9117\se{.0008} & 20.8 & 0.36 & -- & -- & -- \\
\bottomrule
\end{tabular}

\end{table*}

\begin{table*}[t]
  \caption{Gain from context beyond the 2K training window: $\Delta_{>2K}$ = full-prefix CE $-$ 2K-window CE (Section~\ref{sec:setup}); negative = gain. Mean\,$\pm$\,SE over 3 seeds (2 for 369M Samba); \dag\,single seed; --:~not measured; bold: best per column.}
  \label{tab:beyond}
  \centering\small
  \begin{tabular}{l cc cc cc}
\toprule
 & \multicolumn{2}{c}{50M} & \multicolumn{2}{c}{152M} & \multicolumn{2}{c}{369M} \\
\cmidrule(lr){2-3}\cmidrule(lr){4-5}\cmidrule(lr){6-7}
Model & FineWeb-Edu & Pile & FineWeb-Edu & Pile & FineWeb-Edu & Pile \\
\midrule
SISA (ours) & \best{$-$0.028\se{0.002}} & \best{$-$0.050\se{0.005}} & \best{$-$0.021\se{0.003}} & \best{$-$0.055\se{0.002}} & \best{$-$0.017\se{0.002}} & \best{$-$0.048\se{0.004}} \\
SISA w/ RoPE & $+$0.450\se{0.219} & $+$0.535\se{0.264} & $+$0.344\se{0.062} & $+$0.565\se{0.072} & $+$0.039\single & $+$0.153\single \\
FoX + RoPE & $+$0.416\se{0.181} & $+$0.564\se{0.182} & $+$0.643\se{0.128} & $+$0.826\se{0.143} & $+$0.409\se{0.213} & $+$0.579\se{0.189} \\
FoX (NoPE, default) & $+$0.049\se{0.015} & $+$0.050\se{0.029} & $+$0.059\se{0.003} & $+$0.086\se{0.010} & $+$0.143\single & $+$0.197\single \\
Transformer & $+$3.516\se{0.075} & $+$3.877\se{0.083} & $+$4.157\se{0.297} & $+$4.604\se{0.232} & $+$3.706\se{0.111} & $+$4.194\se{0.068} \\
\midrule
Mamba-3 & $-$0.003\se{$<$0.001} & $+$0.009\se{0.001} & $-$0.003\se{0.001} & $+$0.027\se{0.008} & $-$0.006\single & $+$0.010\single \\
Samba & $+$0.099\se{0.003} & $+$0.072\se{0.011} & $+$0.090\se{0.011} & $+$0.109\se{0.023} & $+$0.070\se{0.023} & $+$0.101\se{0.053} \\
Hymba & $+$0.644\se{0.133} & $+$0.844\se{0.212} & $+$0.384\se{0.110} & $+$0.504\se{0.029} & -- & -- \\
\bottomrule
\end{tabular}

\end{table*}

Table~\ref{tab:main} reports held-out CE, LAMBADA accuracy and passkey retrieval at the 2K training length (PK) for the parameter-matched models trained on the same 10B tokens (Section~\ref{sec:setup}), as means over the seeds listed there (with $\pm$\,SE for CE); its upper block holds the attention-only models and its lower block the SSM and hybrid architectures.

\paragraph{Attention baselines.}
\sisa has the lowest CE of all attention-only models at every scale.
At 50M it is $0.0223\pm0.0019$ below \fox, $0.0310\pm0.0008$ below \fox (NoPE) and $0.0360\pm0.0010$ below the Transformer, and at 152M $0.0136\pm0.0005$ below \fox and $0.0219\pm0.0009$ below the Transformer.
All favor \sisa in every seed pair (paired $t$-test, $p<0.01$), i.e.\ 1.4--3.5\% lower perplexity.
The margins narrow with scale; the \sisa-specific settings were chosen at 50M and not retuned.
At 369M, \sisa is $0.0150\pm0.0003$ below the Transformer and $0.0037\pm0.0007$ below \fox, in all three seed pairs of each ($p<0.001$ and $p=0.03$); the latter is about a quarter of the 152M margin.
LAMBADA broadly follows the CE ranking (\sisa best or tied among attention-only models at every scale), but its lead over \fox (0.8--0.9 points) is within two seed standard errors at every scale.

\paragraph{Positional encoding.}
Dropping RoPE adds to \sisa's gain.
Adding RoPE on $\vect{q}^{\top}\vect{k}$ to \sisa, with nothing else changed, raises CE by $0.0126\pm0.0017$ at 50M (Section~\ref{sec:ablation}), and \sisarope, which also lacks $W_C{=}0$ and the main-LR $\lambda_h$ (Section~\ref{sec:setup}), trails \sisa by $0.0102\pm0.0014$ at 152M (both $p<0.02$) and by $0.0084$ at 369M (1 seed).
Without the SSM term, dropping RoPE does not help: \fox (NoPE) is worse than \fox (Table~\ref{tab:main}).

\paragraph{SSM and hybrid baselines.}
Against architectures with recurrent layers, the outcome depends on scale.
\sisa beats Mamba-3 at 50M ($0.0235\pm0.0014$), ties it at 152M ($0.0030\pm0.0020$, $p=0.28$) and trails it at 369M by $0.0116$ (1 seed of a Mamba-3 with added MLP blocks; Section~\ref{sec:setup}).
Samba ties \sisa at 50M (Samba lower by $0.0026\pm0.0020$, $p=0.32$) but leads at 152M by $0.0166\pm0.0005$ and at 369M by $0.0485\pm0.0010$ (paired over Samba's two seeds).
Hymba leads at equal tokens by $0.0156\pm0.0009$ at 50M and $0.0122\pm0.0011$ at 152M.
These hybrids add recurrent layers, whereas \sisa changes only the attention scores.
\sisa is the strongest attention-only model here, the only model with attention we compare that gains from context beyond its training window (Table~\ref{tab:beyond}, Section~\ref{sec:longctx}), and, with our kernels, trains 1.1--1.6$\times$ faster than Samba and 3.3--4.8$\times$ faster than Hymba at 2K context on every GPU where they fit (Section~\ref{sec:efficiency}).

\begin{table}[t]
  \caption{Within Hymba's estimated GPU-time budget at 50M: \sisa's 22.4B tokens use 5--10\% less (3 seeds; $\Delta$ paired).}
  \label{tab:equaltime}
  \small\setlength{\tabcolsep}{3pt}
  \centering
  \resizebox{\columnwidth}{!}{\begin{tabular}{l r c c c c}
\toprule
Model & Tokens & CE$\downarrow$ & $\Delta$ vs.\ Hymba & LAMB.$\uparrow$ & PK$\uparrow$ \\
\midrule
Hymba & 10.0B & 3.2286\se{.0012} & -- & 16.3 & 0.19 \\
SISA (ours), equal tokens & 10.0B & 3.2442\se{.0003} & $+$0.0156\se{.0009} & 16.9 & 0.63 \\
SISA (ours), within time budget & 22.4B & 3.1764\se{.0012} & $-$0.0522\se{.0011} & 17.4 & 0.59 \\
\bottomrule
\end{tabular}
}
\end{table}

\paragraph{Within Hymba's GPU-time budget.}
At 50M, trained on 22.4B tokens, \sisa overtakes Hymba in CE (Table~\ref{tab:equaltime}); Hymba is the hybrid that leads it most at equal tokens at this size.
As \sisa trains faster than Hymba (on the H100, 2.5$\times$ with the reference path and 4.8$\times$ with our kernels; Section~\ref{sec:setup}), \sisa receives the 22.4B-token budget defined there (2.24 passes over the 10B-token training set).
With the reference path, which trained the compared \sisa models, these tokens take an estimated 10\% (H100) and 5\% (A6000) less GPU time than Hymba's 10B, from measured throughput; savings with our kernels are given in Section~\ref{sec:setup}.
\sisa then lowers CE by $0.0522\pm0.0011$ relative to Hymba (all three seed pairs), turning its $0.0156$ deficit at equal tokens into a clear lead, and improves passkey retrieval (0.59 vs.\ 0.19); its LAMBADA gain (17.4\% vs.\ 16.3\%; paired $+1.1\pm1.2$ points) is not significant.

\paragraph{In-context retrieval.}
Among attention-only models, only \sisa combines the lowest CE with top passkey retrieval at the 2K training length (PK in Table~\ref{tab:main}).
It retrieves best among them at 50M and 152M (0.63 and 0.92, the latter 0.01 above \fox (NoPE)) and, at 369M, scores 0.93 against 0.90 for \fox (NoPE) (1 seed) and 0.81 for the Transformer.
Positional encoding on $\vect{q}^{\top}\vect{k}$ hurts retrieval: \sisarope retrieves less than \sisa at every scale (1 seed at 369M), and \fox (NoPE) retrieves more than \fox (0.37, 0.91 and 0.90 vs.\ 0.10, 0.70 and 0.73) but has a higher CE.
Of the models with recurrent layers, Mamba-3 ($\le$0.23) and Hymba ($\le$0.36) retrieve far less, and Samba less often than \sisa at 50M (0.55) but slightly more often at 152M (0.95) and 369M (0.94).

\paragraph{Downstream tasks.}
On ARC-Challenge (24.1--27.8\%), OpenBookQA (28.8--35.0\%), SocialIQA (35.5--39.3\%), WinoGrande (50.8--52.7\%) and BoolQ (53.6--61.7\%), every model evaluated from 50M to 369M (one seed at 369M, except on WinoGrande) scores at most 10 points above chance (25\%, 25\%, 33.3\%, 50\%) and below the BoolQ majority class (62.2\%), and models of the same size differ by at most 3.6 points (7.5 on BoolQ), so we treat these tasks as supplementary.
HellaSwag~\cite{zellers2019hellaswag}, ARC-Easy and PIQA~\cite{bisk2020piqa} rise above chance with scale, but at each size all models lie within 2.7 points of one another.

\section{Length Generalization}\label{sec:longctx}

All models are trained on 2,048-token sequences and evaluated at longer lengths without fine-tuning; RoPE tables are extended without rescaling (base $10^4$).
Position interpolation~\cite{chen2023pi} or YaRN~\cite{peng2024yarn} could help the RoPE models; we apply no length extension to any model, so the decisive comparisons beyond 2K are with the models that need none: \fox (NoPE), Mamba-3 and Samba, whose 2K attention window never exceeds the trained distances.

\begin{figure}[t]
  \centering
  \includegraphics[width=\columnwidth]{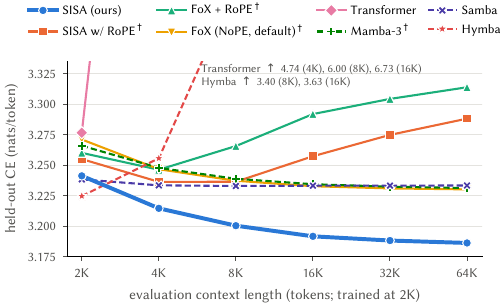}
  \Description{Line plot of held-out cross-entropy against evaluation length, from 2K to 64K tokens on a logarithmic axis, for eight 50M models trained at 2K. The SISA curve falls steadily and is the lowest from 4K on. FoX (NoPE), Mamba-3 and Samba stay nearly flat, about 0.04 nats above SISA at 64K. FoX + RoPE and SISA w/ RoPE rise at longer lengths, and the Transformer and Hymba rise steeply and leave the plotted range.}
  \caption{Held-out CE vs.\ length at 50M, 2K training (3-seed mean; \dag\,seed 1 beyond 16K; \fox (NoPE): seed 1 throughout).}
  \label{fig:ce-length}
\end{figure}

\begin{table}[t]
  \caption{Passkey accuracy at 152M by evaluation length (mean\,$\pm$\,SE over 3 seeds).}
  \label{tab:passkey}
  \centering\small\setlength{\tabcolsep}{3.5pt}%
  \begin{tabular}{l cccc}
\toprule
Model & 2K (1$\times$) & 4K (2$\times$) & 8K (4$\times$) & 16K (8$\times$) \\
\midrule
SISA (ours) & 0.92\se{0.01} & \best{0.86\se{0.04}} & \best{0.41\se{0.09}} & \best{0.10\se{0.05}} \\
SISA w/ RoPE & 0.63\se{0.08} & 0.03\se{0.01} & 0.00\se{0.00} & 0.00\se{0.00} \\
FoX + RoPE & 0.70\se{0.10} & 0.04\se{0.03} & 0.00\se{0.00} & 0.00\se{0.00} \\
FoX (NoPE, default) & 0.91\se{0.01} & 0.37\se{0.06} & 0.02\se{0.01} & 0.00\se{0.00} \\
Transformer & 0.63\se{0.05} & 0.00\se{0.00} & 0.00\se{0.00} & 0.00\se{0.00} \\
\midrule
Mamba-3 & 0.06\se{0.02} & 0.00\se{0.00} & 0.00\se{0.00} & 0.00\se{0.00} \\
Samba & \best{0.95\se{0.01}} & 0.22\se{0.07} & 0.00\se{0.00} & 0.00\se{0.00} \\
Hymba & 0.36\se{0.05} & 0.00\se{0.00} & 0.00\se{0.00} & 0.00\se{0.00} \\
\bottomrule
\end{tabular}

\end{table}

\paragraph{Stability up to 64K tokens.}
At 50M, \sisa's CE falls at every length up to 64K tokens, from 3.2413 to 3.1863 (SE $\le$\,0.0005), and is the lowest of all models from 4K on, although Hymba and Samba lead at 2K (Fig.~\ref{fig:ce-length}).
\fox (NoPE), Mamba-3 and Samba (as its authors report~\cite{ren2025samba}) do not degrade either but end 0.044--0.047 above \sisa at 64K (\fox (NoPE) and Mamba-3: seed 1).
The RoPE-based attention models and Hymba degrade at longer lengths: \fox + RoPE worsens by 0.054 nats by 64K (seed 1), the Transformer collapses, and Hymba rises steeply (the Transformer's and Hymba's values above the plotted range are annotated at the top of Fig.~\ref{fig:ce-length}).
Fig.~\ref{fig:ce-length} scores the same 2.1M held-out FineWeb-Edu tokens at every length in non-overlapping windows (Table~\ref{tab:main} uses 4.2M tokens, hence slightly different values at 2K); curves marked \dag{} have one seed beyond 16K (\fox (NoPE): at every length), and the Transformer and Hymba were run only up to 16K.
Because the windows do not overlap, longer windows also hold fewer tokens with a short prefix, so the figure shows stability rather than the use of distant context.

\paragraph{Using context beyond the training window.}
\sisa is the only model with attention we compare whose loss falls with context more than 2K tokens back, in every setting measured (Table~\ref{tab:beyond}, which scores the same targets with the full prefix and with only the preceding 2,048 tokens).
At 50M, its $\Delta_{>2K}$ is $-0.028\pm0.002$ on FineWeb-Edu and $-0.050\pm0.005$ on the Pile, and even its smallest gain (369M, FineWeb-Edu) is $-0.017\pm0.002$.
The extra context hurts every RoPE-based model and Hymba (Table~\ref{tab:beyond}), and also the two other models with attention that need no length extension: \fox (NoPE) ($+0.05$ at 50M; $+0.06$/$+0.09$ at 152M; $+0.14$/$+0.20$ at 369M, 1 seed) and Samba ($+0.07$ to $+0.11$).
The attention-free Mamba-3 gains slightly on FineWeb-Edu only (0.003 nats at 50M and 152M; 0.006 at 369M, one seed) and loses on the Pile.

\paragraph{Passkey retrieval.}
At twice the training length, \sisa retrieves more passkeys than any baseline measured, at all three scales (152M in Table~\ref{tab:passkey}, 50M and 369M below; protocol in Section~\ref{sec:setup}).
At 152M, Samba, \sisa and \fox (NoPE) retrieve 0.91--0.95 of the keys at 2K, but at 4K \sisa still retrieves $0.86\pm0.04$, against $0.37\pm0.06$ for \fox (NoPE) and at most 0.22 for every other model, and at 8K it retrieves $0.41\pm0.09$ against at most 0.02.
The same holds at 50M ($0.30\pm0.04$ at 4K vs.\ at most 0.14) and at 369M ($0.85\pm0.02$ at 4K vs.\ 0.57 for \fox (NoPE), 1 seed, and at most 0.30 for the others; $0.40\pm0.02$ at 8K vs.\ at most 0.01).

\paragraph{Why \sisa extrapolates: a hypothesis.}
We hypothesize that the rotation supplies local order and the forget gate long-range recency.
In Eq.~\eqref{eq:score}, no explicit positional encoding is applied to $\vect{q}_i^{\top}\vect{k}_j$: explicit order and recency signals enter only through the rotation $\Phi_i-\Phi_j$ and the decay $e^{g_i-g_j}$ of the SSM term and through the gate term $F_i-F_j$, and the decay and the gate only decrease as $j$ recedes.
Since the SSM term is small beyond about 128 tokens in our 2K-token analysis (Section~\ref{sec:ablation}), the rotation would act mainly on nearby keys.
The rotation acts only in this decaying channel, as in \mbox{Mamba-3}'s recurrence, so distant keys would meet angles unseen in training only at small weight, unlike under RoPE.
The gate penalizes distance in the same monotone way at any length, as in ALiBi~\cite{press2022alibi}.
The ablations are consistent with this account but do not isolate it: adding RoPE back or setting $\Phi=0$ erases the beyond-window gain, and \fox (NoPE), which has the gate but no SSM term, is stable but gains nothing (Section~\ref{sec:ablation}).

\section{Training Efficiency}\label{sec:efficiency}

\begin{figure}[t]
  \centering
  \includegraphics[width=\columnwidth]{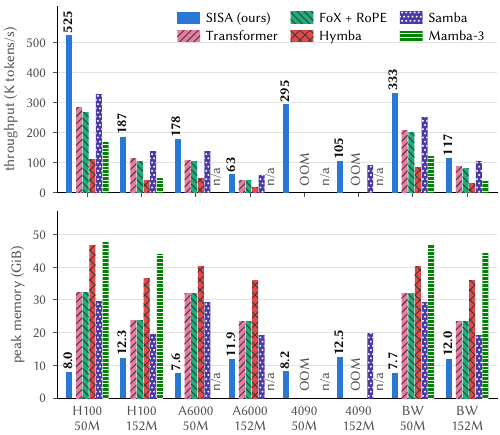}
  \Description{Two bar charts stacked vertically share one horizontal axis with eight groups: the H100, RTX A6000, RTX 4090 and RTX PRO 5000 Blackwell (BW) GPUs, each at 50M and 152M parameters. Each group has one bar each for SISA, the Transformer, FoX with RoPE, Hymba, Samba and Mamba-3. The top chart shows training throughput in thousands of tokens per second and the bottom chart peak GPU memory in GiB. In every group SISA has the tallest throughput bar and the shortest memory bar, for example 525 thousand tokens per second and 8.0 GiB on the H100 at 50M, against 285 thousand and 32.2 GiB for the Transformer, and 63 thousand tokens per second and 11.9 GiB on the RTX A6000 at 152M, against 57 thousand and 19.2 GiB for Samba. On the RTX 4090 the Transformer, FoX and Hymba at both sizes and Samba at 50M are marked out of memory, and Mamba-3 is marked not measured on the RTX A6000 and RTX 4090.}
  \caption{Training throughput (top) and peak memory (bottom) at a 2K context (BW: Blackwell; n/a: not measured).}
  \label{fig:efficiency}
\end{figure}

With our kernels, \sisa trains faster and uses less memory than every baseline measured, at both sizes on each of the four GPUs (Fig.~\ref{fig:efficiency}), at a 2K training context and the same micro-batch (protocol in Section~\ref{sec:setup}).
Throughput is the steady-state rate after warm-up and memory the peak reported by nvidia-smi (polled every 200\,ms, so including the CUDA context and the allocator's reserve); \fox (NoPE), omitted from the figure, is within 0.3\% of \fox in both.
Across the GPUs and sizes where each baseline runs, its throughput is 1.3--1.8$\times$ that of the Transformer, 1.4--2.0$\times$ that of \fox, 1.1--1.6$\times$ that of Samba, 2.7--3.9$\times$ that of Mamba-3 and 3.3--4.8$\times$ that of Hymba, and it peaks at 7.6--8.2\,GiB at 50M and 11.9--12.5\,GiB at 152M, while the baselines need 1.6--6.0$\times$ more.
The throughput margin is smallest over Samba at 152M on the A6000: 1.10$\times$ its throughput (63K vs.\ 57K tokens/s) with 1.62$\times$ less memory.

The baselines run their official or standard model code in our training loop in eager PyTorch (no torch.compile, unfused AdamW, loss read back every micro-batch) and compute the loss on full fp32 logits, whereas our kernels use a fused AdamW, GPU-side loss accumulation, a chunked LM head and loss, and CUDA graphs (Section~\ref{sec:method:kernels}).
On the H100, \sisa, the Transformer, \fox and Mamba-3 share one software stack and Samba and Hymba use that of their official kernels; on the A6000 and the Blackwell, \sisa runs newer PyTorch and Triton than the baselines; the ratios over the Transformer and \fox are largest on the H100, where they share \sisa's stack.\footnote{PyTorch/Triton: H100 2.11/3.6.0 (Samba, Hymba: 2.5/3.1.0); A6000 2.12/3.7.1 for \sisa, 2.5/3.1.0 for the baselines; RTX 4090 2.5/3.1.0; Blackwell 2.14/3.8.0 for \sisa, 2.9/3.5.0 for the baselines.}
Each model is run once per GPU and size, \sisa and the Transformer two or more times (we report the mean); repeats differ by at most 0.5\%.
Mamba-3 is measured only on the H100 and the Blackwell: at its default chunk size, its official backward kernels exceed the per-block shared memory of the A6000, RTX 4090 and Blackwell, and we ran it at chunk size 32 on the Blackwell only.
On the Blackwell, \fox's official kernel runs with its tile sizes for the 4090's architecture, as it has none for the Blackwell.

On the 24\,GB RTX 4090, at this micro-batch (16 sequences at 50M, 8 at 152M), the Transformer, both \fox variants and Hymba run out of memory at both sizes and Samba at 50M (OOM in Fig.~\ref{fig:efficiency}), each when allocating a full-vocabulary fp32 logit tensor (6.1 or 3.1\,GiB) in its standard loss or its backward pass. Two of our 152M \fox (NoPE) seeds were nevertheless trained on this GPU at micro-batch 4 with gradient accumulation.
At 152M, Samba trains at 92K tokens/s in 19.7\,GiB, vs.\ 105K in 12.5\,GiB for \sisa.

Fig.~\ref{fig:efficiency} therefore compares complete training implementations, not attention rules in isolation: the SSM term adds a quarter more matrix-multiply work to the forward attention core (Section~\ref{sec:method:chunk}), so the margins come from our kernels.
Every reported \sisa model was trained with the reference path (except the 369M \sisarope, trained with an early version of our kernels), which computes the same function up to floating-point rounding (Section~\ref{sec:method:kernels}); on the H100, A6000 and 4090, our kernels are 1.5--1.9$\times$ faster and need 1.4--2.2$\times$ less memory (H100 reference path: 274K and 99K tokens/s in 14.5 and 16.9\,GiB at 50M and 152M).
The reference path is about as fast as the Transformer and \fox at 50M (0.96--1.12$\times$ their throughput) but slower at 152M (0.86--0.96$\times$), and slower than Samba (0.69--0.85$\times$); it still needs less memory than every baseline.

\section{Ablations}\label{sec:ablation}

Table~\ref{tab:ablation} lists ablations at 50M, where all design choices were made; each variant is trained with three seeds on the same data, in the same order, as \sisa (training setup in Section~\ref{sec:setup}).
$\Delta$CE is a variant's CE minus \sisa's, paired over seeds ($\pm$SE); PK is passkey accuracy at 2K and 4K tokens; $\Delta_{>2K}$ is measured on FineWeb-Edu (Section~\ref{sec:setup}; negative = gain).
\sisarope uses the base recipe at every size (Section~\ref{sec:setup}): besides adding RoPE, it lacks $W_C{=}0$ and the main-LR $\lambda_h$.

\begin{table}[t]
  \caption{Ablations at 50M (mean\,$\pm$\,SE over 3 seeds; $\Delta$CE: variant $-$ \sisa, paired; \dag\,single seed; --:\,not measured).}
  \label{tab:ablation}
  \centering
  \small
  \setlength{\tabcolsep}{2.5pt}
  \resizebox{\columnwidth}{!}{\begin{tabular}{l c c c c c}
\toprule
Variant & CE$\downarrow$ & $\Delta$CE & PK 2K & PK 4K & $\Delta_{>2K}$ \\
\midrule
SISA (no RoPE on $\vect{q}^{\top}\vect{k}$; full) & 3.2442\se{.0003} & -- & 0.63 & 0.30 & $-$0.028\se{0.002} \\
\quad w/o rotation $\Phi$ & 3.2585\se{.0006} & $+$0.0143\se{.0008} & 0.45 & 0.16 & $+$0.004\se{0.009} \\
\quad default $W_C$ init, $\lambda_h$ LR $6{\times}10^{-5}$ & 3.2480\se{.0007} & $+$0.0038\se{.0010} & 0.47 & 0.17\single & $-$0.023\single \\
\quad + RoPE on $\vect{q}^{\top}\vect{k}$ & 3.2568\se{.0015} & $+$0.0126\se{.0017} & 0.11 & 0.00 & $+$0.202\se{0.008} \\
\midrule
SISA w/ RoPE (base recipe) & 3.2595\se{.0013} & $+$0.0153\se{.0016} & 0.14 & 0.00 & $+$0.450\se{0.219} \\
\quad w/o rotation $\Phi$ & 3.2610\se{.0006} & $+$0.0168\se{.0007} & 0.37 & -- & -- \\
\midrule
FoX + RoPE & 3.2666\se{.0016} & $+$0.0223\se{.0019} & 0.10 & 0.00 & $+$0.416\se{0.181} \\
FoX (NoPE, default) & 3.2752\se{.0005} & $+$0.0310\se{.0008} & 0.37 & 0.07\single & $+$0.049\se{0.015} \\
\bottomrule
\end{tabular}
}
\end{table}

\paragraph{The rotated SSM term is the core.}
\fox (NoPE), which has \sisa's forget gate but no SSM term, trails \sisa by $0.0310\pm0.0008$.
Setting only the rotation to $\Phi = 0$, which keeps the score width and parameter count, still costs $0.0143\pm0.0008$, removes the beyond-window gain ($\Delta_{>2K}$ $+0.004\pm0.009$) and lowers passkey accuracy at 4K from $0.30\pm0.04$ to $0.16\pm0.02$.

\paragraph{The SSM term as input-dependent importance.}
The SSM term is local: on 16 held-out 2K-token sequences (all layers and heads; 50M: 3 seeds, 152M/369M: seed~1), its mean absolute value over query--key pairs is 1.5--2.5 logits at distance 0--1 and below 0.01 logits beyond 128 tokens.
It varies with the input: over each query's last 64 keys, its standard deviation is 0.63/0.49/0.46 times the content term's at 50M/152M/369M, and distance explains only 38--51\% of its variance (content term: 5\%).
On keys weighted by attention computed without it, it exceeds its row mean by only 0.15--0.20 logits, against 1.7--2.1 for $\vect{q}^{\top}\vect{k}$, which, with the forget gate, drives that attention; this is consistent with a contribution distinct from the content term's selection.

\paragraph{Given $\Phi$, adding RoPE hurts.}
In the base recipe with RoPE on $\vect{q}^{\top}\vect{k}$, removing $\Phi$ has little effect on held-out CE ($+0.0015\pm0.0019$), consistent with the two encodings carrying overlapping positional information; passkey accuracy at 2K, noisy across seeds, moves from $0.14\pm0.08$ to $0.37\pm0.04$ (paired $p=0.11$).
Conversely, adding only RoPE to the final model raises CE by $0.0126\pm0.0017$ and breaks length generalization: $\Delta_{>2K}$ turns positive, $+0.202\pm0.008$ on FineWeb-Edu and $+0.347\pm0.093$ on the Pile, and passkey accuracy at 4K falls to $0.00$.
Removing only RoPE from \sisarope lowers CE by $0.0115\pm0.0008$ of its $0.0153\pm0.0016$ gap to \sisa; the rest is the base recipe.

\paragraph{Initialization and $\lambda_h$ learning rate are minor.}
$W_C{=}0$ and the main-LR $\lambda_h$ (instead of $6\times10^{-5}$) lower CE by $0.0038\pm0.0010$ without RoPE and by an insignificant $0.0027\pm0.0025$ with RoPE; the variant without them still gains beyond 2K ($-0.023$\single).

\paragraph{In our recipe, a forget gate alone does not replace RoPE}%
\fox shares the recency term $F_i - F_j$ with \sisa, yet removing its RoPE raises CE at every scale ($+0.0086\pm0.0011$ at 50M; $+0.010\pm0.002$ at 152M and $+0.007$ at 369M with one \fox (NoPE) seed) and yields no gain beyond 2K ($\Delta_{>2K}$ $+0.049\pm0.015$ at 50M); \fox's authors likewise find RoPE slightly helpful in this LLaMA-style block, though not in their Pro block~\cite{lin2025fox}.
Thus, dropping RoPE hurt \fox at all three scales (369M: one \fox (NoPE) seed) but helped \sisa, whose SSM term is rotated.

\paragraph{Transfer to larger scales.}
Of the \sisa ablations, only \sisarope (base recipe) was trained above 50M.
It trails \sisa at every size, by $0.0102\pm0.0014$ at 152M and $0.0084$ at 369M (one seed), and at 152M it is still hurt by context beyond 2K ($\Delta_{>2K}$ $+0.344\pm0.062$ on FineWeb-Edu; Table~\ref{tab:beyond}).

\section{Limitations and Conclusion}\label{sec:conclusion}

\sisa builds Mamba-3-style state transitions, rotation included, directly into the attention score, alongside \fox's forget gate; since the rotation encodes relative position, \sisa drops RoPE, and adding it back hurts in our ablations.
This yields stronger attention: with matched parameters and data, \sisa has the lowest held-out loss among attention-only models at every scale we trained (50M, 152M and 369M) and is the only model with attention we compare whose loss falls with context beyond its 2K training window.
With our kernels, it also trains faster and with less memory than every baseline measured at a 2K context on four GPUs (50M and 152M): 1.3--1.8$\times$ the Transformer's throughput with 1.9--4.2$\times$ less peak memory, and 1.1--1.6$\times$ the throughput of Samba, the fastest baseline.

\paragraph{Limitations.}
\emph{Scale and seeds.} Our models span 50M to 369M parameters. \sisa's in-window margin over \fox narrows from $0.022$ at 50M to $0.004\pm0.001$ at 369M, whereas its beyond-window gain persists at every scale and on both corpora ($-0.017\pm0.002$ and $-0.048\pm0.004$ at 369M; Table~\ref{tab:beyond}); whether its lead holds at billion-parameter scale is open. At 369M, Samba has two seeds and Mamba-3, \fox (NoPE) and \sisarope one.
\emph{Length generalization.} Passkey retrieval fades with distance (Section~\ref{sec:longctx}); we measured CE beyond 16K only at 50M and gave the RoPE models no training-free extension such as NTK-aware scaling or Dynamic YaRN~\cite{peng2024yarn}; without a static-rotation or gate-free ablation, whether the gain requires an input-dependent rotation or the forget gate is untested.
\emph{Evaluation.} At our scale, zero-shot benchmarks are supplementary (Section~\ref{sec:results}); our conclusions rest on held-out loss, LAMBADA, passkey retrieval and long-context loss.
\emph{Efficiency.} Our timings use our kernels, whereas the reported \sisa models were trained with the numerically equivalent reference path (369M \sisarope: an early version of our kernels; Section~\ref{sec:efficiency}); we time training at a 2K context but not decoding, for which \sisa caches $d_s+2=34$ values per head and token beyond keys and values (+27\%).

\paragraph{Future work.}
As \sisa changes only the attention scores, it could replace the attention inside hybrids such as Samba and Hymba (their sliding-window attention or attention heads), which do not gain from context beyond 2K (Table~\ref{tab:beyond}); training beyond 369M parameters, where zero-shot benchmarks may begin to separate the models, is the other next step.
Code and kernels are available at \url{https://github.com/tlstngud/sisa}.

\bibliographystyle{ACM-Reference-Format}
\bibliography{refs}


\begin{thebibliography}{36}


\ifx \showCODEN    \undefined \def \showCODEN     #1{\unskip}     \fi
\ifx \showISBNx    \undefined \def \showISBNx     #1{\unskip}     \fi
\ifx \showISBNxiii \undefined \def \showISBNxiii  #1{\unskip}     \fi
\ifx \showISSN     \undefined \def \showISSN      #1{\unskip}     \fi
\ifx \showLCCN     \undefined \def \showLCCN      #1{\unskip}     \fi
\ifx \shownote     \undefined \def \shownote      #1{#1}          \fi
\ifx \showarticletitle \undefined \def \showarticletitle #1{#1}   \fi
\ifx \showURL      \undefined \def \showURL       {\relax}        \fi
\providecommand\bibfield[2]{#2}
\providecommand\bibinfo[2]{#2}
\providecommand\natexlab[1]{#1}
\providecommand\showeprint[2][]{arXiv:#2}

\bibitem[Bisk et~al\mbox{.}(2020)]%
        {bisk2020piqa}
\bibfield{author}{\bibinfo{person}{Yonatan Bisk}, \bibinfo{person}{Rowan
  Zellers}, \bibinfo{person}{Ronan~Le Bras}, \bibinfo{person}{Jianfeng Gao},
  {and} \bibinfo{person}{Yejin Choi}.} \bibinfo{year}{2020}\natexlab{}.
\newblock \showarticletitle{{PIQA}: Reasoning about Physical Commonsense in
  Natural Language}. In \bibinfo{booktitle}{\emph{Proceedings of the
  Thirty-Fourth AAAI Conference on Artificial Intelligence (AAAI 2020)}}.
  \bibinfo{pages}{7432--7439}.
\newblock


\bibitem[Black et~al\mbox{.}(2022)]%
        {black2022gptneox}
\bibfield{author}{\bibinfo{person}{Sid Black}, \bibinfo{person}{Stella
  Biderman}, \bibinfo{person}{Eric Hallahan}, \bibinfo{person}{Quentin
  Anthony}, \bibinfo{person}{Leo Gao}, \bibinfo{person}{Laurence Golding},
  \bibinfo{person}{Horace He}, \bibinfo{person}{Connor Leahy},
  \bibinfo{person}{Kyle McDonell}, \bibinfo{person}{Jason Phang},
  \bibinfo{person}{Michael Pieler}, \bibinfo{person}{USVSN~Sai Prashanth},
  \bibinfo{person}{Shivanshu Purohit}, \bibinfo{person}{Laria Reynolds},
  \bibinfo{person}{Jonathan Tow}, \bibinfo{person}{Ben Wang}, {and}
  \bibinfo{person}{Samuel Weinbach}.} \bibinfo{year}{2022}\natexlab{}.
\newblock \showarticletitle{{GPT-NeoX-20B}: An Open-Source Autoregressive
  Language Model}. In \bibinfo{booktitle}{\emph{Proceedings of BigScience
  Episode \#5 -- Workshop on Challenges \& Perspectives in Creating Large
  Language Models}}. \bibinfo{pages}{95--136}.
\newblock
\href{https://doi.org/10.18653/v1/2022.bigscience-1.9}{doi:\nolinkurl{10.18653/v1/2022.bigscience-1.9}}


\bibitem[Chen et~al\mbox{.}(2023)]%
        {chen2023pi}
\bibfield{author}{\bibinfo{person}{Shouyuan Chen}, \bibinfo{person}{Sherman
  Wong}, \bibinfo{person}{Liangjian Chen}, {and} \bibinfo{person}{Yuandong
  Tian}.} \bibinfo{year}{2023}\natexlab{}.
\newblock \showarticletitle{Extending Context Window of Large Language Models
  via Positional Interpolation}.
\newblock \bibinfo{journal}{\emph{arXiv preprint arXiv:2306.15595}}
  (\bibinfo{year}{2023}).
\newblock


\bibitem[Clark et~al\mbox{.}(2019)]%
        {clark2019boolq}
\bibfield{author}{\bibinfo{person}{Christopher Clark}, \bibinfo{person}{Kenton
  Lee}, \bibinfo{person}{Ming-Wei Chang}, \bibinfo{person}{Tom Kwiatkowski},
  \bibinfo{person}{Michael Collins}, {and} \bibinfo{person}{Kristina
  Toutanova}.} \bibinfo{year}{2019}\natexlab{}.
\newblock \showarticletitle{{BoolQ}: Exploring the Surprising Difficulty of
  Natural Yes/No Questions}. In \bibinfo{booktitle}{\emph{Proceedings of the
  2019 Conference of the North American Chapter of the Association for
  Computational Linguistics (NAACL-HLT 2019)}}. \bibinfo{pages}{2924--2936}.
\newblock
\href{https://doi.org/10.18653/v1/N19-1300}{doi:\nolinkurl{10.18653/v1/N19-1300}}


\bibitem[Clark et~al\mbox{.}(2018)]%
        {clark2018arc}
\bibfield{author}{\bibinfo{person}{Peter Clark}, \bibinfo{person}{Isaac
  Cowhey}, \bibinfo{person}{Oren Etzioni}, \bibinfo{person}{Tushar Khot},
  \bibinfo{person}{Ashish Sabharwal}, \bibinfo{person}{Carissa Schoenick},
  {and} \bibinfo{person}{Oyvind Tafjord}.} \bibinfo{year}{2018}\natexlab{}.
\newblock \showarticletitle{Think you have Solved Question Answering? {Try}
  {ARC}, the {AI2} Reasoning Challenge}.
\newblock \bibinfo{journal}{\emph{arXiv preprint arXiv:1803.05457}}
  (\bibinfo{year}{2018}).
\newblock


\bibitem[Dao(2024)]%
        {dao2024flashattention2}
\bibfield{author}{\bibinfo{person}{Tri Dao}.} \bibinfo{year}{2024}\natexlab{}.
\newblock \showarticletitle{{FlashAttention-2}: Faster Attention with Better
  Parallelism and Work Partitioning}. In \bibinfo{booktitle}{\emph{The Twelfth
  International Conference on Learning Representations (ICLR 2024)}}.
\newblock


\bibitem[Dao et~al\mbox{.}(2022)]%
        {dao2022flashattention}
\bibfield{author}{\bibinfo{person}{Tri Dao}, \bibinfo{person}{Daniel~Y. Fu},
  \bibinfo{person}{Stefano Ermon}, \bibinfo{person}{Atri Rudra}, {and}
  \bibinfo{person}{Christopher R{\'e}}.} \bibinfo{year}{2022}\natexlab{}.
\newblock \showarticletitle{{FlashAttention}: Fast and Memory-Efficient Exact
  Attention with {IO}-Awareness}. In \bibinfo{booktitle}{\emph{Advances in
  Neural Information Processing Systems 35 (NeurIPS 2022)}}.
\newblock


\bibitem[Dao and Gu(2024)]%
        {dao2024mamba2}
\bibfield{author}{\bibinfo{person}{Tri Dao} {and} \bibinfo{person}{Albert Gu}.}
  \bibinfo{year}{2024}\natexlab{}.
\newblock \showarticletitle{Transformers are {SSMs}: Generalized Models and
  Efficient Algorithms Through Structured State Space Duality}. In
  \bibinfo{booktitle}{\emph{Proceedings of the 41st International Conference on
  Machine Learning (ICML 2024)}} \emph{(\bibinfo{series}{Proceedings of Machine
  Learning Research}, Vol.~\bibinfo{volume}{235})}.
  \bibinfo{pages}{10041--10071}.
\newblock


\bibitem[Dong et~al\mbox{.}(2025)]%
        {dong2025hymba}
\bibfield{author}{\bibinfo{person}{Xin Dong}, \bibinfo{person}{Yonggan Fu},
  \bibinfo{person}{Shizhe Diao}, \bibinfo{person}{Wonmin Byeon},
  \bibinfo{person}{Zijia Chen}, \bibinfo{person}{Ameya~Sunil Mahabaleshwarkar},
  \bibinfo{person}{Shih-Yang Liu}, \bibinfo{person}{Matthijs~Van Keirsbilck},
  \bibinfo{person}{Min-Hung Chen}, \bibinfo{person}{Yoshi Suhara},
  \bibinfo{person}{Yingyan~Celine Lin}, \bibinfo{person}{Jan Kautz}, {and}
  \bibinfo{person}{Pavlo Molchanov}.} \bibinfo{year}{2025}\natexlab{}.
\newblock \showarticletitle{Hymba: A Hybrid-head Architecture for Small
  Language Models}. In \bibinfo{booktitle}{\emph{The Thirteenth International
  Conference on Learning Representations (ICLR 2025)}}.
\newblock


\bibitem[Gao et~al\mbox{.}(2020)]%
        {gao2020pile}
\bibfield{author}{\bibinfo{person}{Leo Gao}, \bibinfo{person}{Stella Biderman},
  \bibinfo{person}{Sid Black}, \bibinfo{person}{Laurence Golding},
  \bibinfo{person}{Travis Hoppe}, \bibinfo{person}{Charles Foster},
  \bibinfo{person}{Jason Phang}, \bibinfo{person}{Horace He},
  \bibinfo{person}{Anish Thite}, \bibinfo{person}{Noa Nabeshima},
  \bibinfo{person}{Shawn Presser}, {and} \bibinfo{person}{Connor Leahy}.}
  \bibinfo{year}{2020}\natexlab{}.
\newblock \showarticletitle{The {Pile}: An 800{GB} Dataset of Diverse Text for
  Language Modeling}.
\newblock \bibinfo{journal}{\emph{arXiv preprint arXiv:2101.00027}}
  (\bibinfo{year}{2020}).
\newblock


\bibitem[Gao et~al\mbox{.}(2023)]%
        {gao2024lmeval}
\bibfield{author}{\bibinfo{person}{Leo Gao}, \bibinfo{person}{Jonathan Tow},
  \bibinfo{person}{Baber Abbasi}, \bibinfo{person}{Stella Biderman},
  \bibinfo{person}{Sid Black}, \bibinfo{person}{Anthony DiPofi},
  \bibinfo{person}{Charles Foster}, \bibinfo{person}{Laurence Golding},
  \bibinfo{person}{Jeffrey Hsu}, \bibinfo{person}{Alain~Le Noac'h},
  \bibinfo{person}{Haonan Li}, \bibinfo{person}{Kyle McDonell},
  \bibinfo{person}{Niklas Muennighoff}, \bibinfo{person}{Chris Ociepa},
  \bibinfo{person}{Jason Phang}, \bibinfo{person}{Laria Reynolds},
  \bibinfo{person}{Hailey Schoelkopf}, \bibinfo{person}{Aviya Skowron},
  \bibinfo{person}{Lintang Sutawika}, \bibinfo{person}{Eric Tang},
  \bibinfo{person}{Anish Thite}, \bibinfo{person}{Ben Wang},
  \bibinfo{person}{Kevin Wang}, {and} \bibinfo{person}{Andy Zou}.}
  \bibinfo{year}{2023}\natexlab{}.
\newblock \bibinfo{title}{A Framework for Few-Shot Language Model Evaluation}.
\newblock \bibinfo{howpublished}{Zenodo, version v0.4.0}.
\newblock
\href{https://doi.org/10.5281/zenodo.10256836}{doi:\nolinkurl{10.5281/zenodo.10256836}}


\bibitem[Golovneva et~al\mbox{.}(2024)]%
        {golovneva2024cope}
\bibfield{author}{\bibinfo{person}{Olga Golovneva}, \bibinfo{person}{Tianlu
  Wang}, \bibinfo{person}{Jason Weston}, {and} \bibinfo{person}{Sainbayar
  Sukhbaatar}.} \bibinfo{year}{2024}\natexlab{}.
\newblock \showarticletitle{Contextual Position Encoding: Learning to Count
  What's Important}.
\newblock \bibinfo{journal}{\emph{arXiv preprint arXiv:2405.18719}}
  (\bibinfo{year}{2024}).
\newblock


\bibitem[Gu and Dao(2024)]%
        {gu2024mamba}
\bibfield{author}{\bibinfo{person}{Albert Gu} {and} \bibinfo{person}{Tri Dao}.}
  \bibinfo{year}{2024}\natexlab{}.
\newblock \showarticletitle{Mamba: Linear-Time Sequence Modeling with Selective
  State Spaces}. In \bibinfo{booktitle}{\emph{First Conference on Language
  Modeling (COLM 2024)}}.
\newblock


\bibitem[Haviv et~al\mbox{.}(2022)]%
        {haviv2022nope}
\bibfield{author}{\bibinfo{person}{Adi Haviv}, \bibinfo{person}{Ori Ram},
  \bibinfo{person}{Ofir Press}, \bibinfo{person}{Peter Izsak}, {and}
  \bibinfo{person}{Omer Levy}.} \bibinfo{year}{2022}\natexlab{}.
\newblock \showarticletitle{Transformer Language Models without Positional
  Encodings Still Learn Positional Information}. In
  \bibinfo{booktitle}{\emph{Findings of the Association for Computational
  Linguistics: EMNLP 2022}}. \bibinfo{pages}{1382--1390}.
\newblock
\href{https://doi.org/10.18653/v1/2022.findings-emnlp.99}{doi:\nolinkurl{10.18653/v1/2022.findings-emnlp.99}}


\bibitem[Kazemnejad et~al\mbox{.}(2023)]%
        {kazemnejad2023nope}
\bibfield{author}{\bibinfo{person}{Amirhossein Kazemnejad},
  \bibinfo{person}{Inkit Padhi}, \bibinfo{person}{Karthikeyan~Natesan
  Ramamurthy}, \bibinfo{person}{Payel Das}, {and} \bibinfo{person}{Siva
  Reddy}.} \bibinfo{year}{2023}\natexlab{}.
\newblock \showarticletitle{The Impact of Positional Encoding on Length
  Generalization in Transformers}. In \bibinfo{booktitle}{\emph{Advances in
  Neural Information Processing Systems 36 (NeurIPS 2023)}}.
\newblock


\bibitem[Lahoti et~al\mbox{.}(2026)]%
        {lahoti2026mamba3}
\bibfield{author}{\bibinfo{person}{Aakash Lahoti}, \bibinfo{person}{Kevin~Y.
  Li}, \bibinfo{person}{Berlin Chen}, \bibinfo{person}{Caitlin Wang},
  \bibinfo{person}{Aviv Bick}, \bibinfo{person}{J.~Zico Kolter},
  \bibinfo{person}{Tri Dao}, {and} \bibinfo{person}{Albert Gu}.}
  \bibinfo{year}{2026}\natexlab{}.
\newblock \showarticletitle{{Mamba-3}: Improved Sequence Modeling using State
  Space Principles}. In \bibinfo{booktitle}{\emph{The Fourteenth International
  Conference on Learning Representations (ICLR 2026)}}.
\newblock
\shownote{arXiv:2603.15569}.
\newblock


\bibitem[Lieber et~al\mbox{.}(2024)]%
        {lieber2024jamba}
\bibfield{author}{\bibinfo{person}{Opher Lieber}, \bibinfo{person}{Barak Lenz},
  \bibinfo{person}{Hofit Bata}, \bibinfo{person}{Gal Cohen},
  \bibinfo{person}{Jhonathan Osin}, \bibinfo{person}{Itay Dalmedigos},
  \bibinfo{person}{Erez Safahi}, \bibinfo{person}{Shaked Meirom},
  \bibinfo{person}{Yonatan Belinkov}, \bibinfo{person}{Shai Shalev-Shwartz},
  \bibinfo{person}{Omri Abend}, \bibinfo{person}{Raz Alon},
  \bibinfo{person}{Tomer Asida}, \bibinfo{person}{Amir Bergman},
  {et~al\mbox{.}}} \bibinfo{year}{2024}\natexlab{}.
\newblock \showarticletitle{Jamba: A Hybrid {Transformer-Mamba} Language
  Model}.
\newblock \bibinfo{journal}{\emph{arXiv preprint arXiv:2403.19887}}
  (\bibinfo{year}{2024}).
\newblock


\bibitem[Lin et~al\mbox{.}(2025)]%
        {lin2025fox}
\bibfield{author}{\bibinfo{person}{Zhixuan Lin}, \bibinfo{person}{Evgenii
  Nikishin}, \bibinfo{person}{Xu~Owen He}, {and} \bibinfo{person}{Aaron
  Courville}.} \bibinfo{year}{2025}\natexlab{}.
\newblock \showarticletitle{Forgetting Transformer: Softmax Attention with a
  Forget Gate}. In \bibinfo{booktitle}{\emph{The Thirteenth International
  Conference on Learning Representations (ICLR 2025)}}.
\newblock


\bibitem[Loshchilov and Hutter(2019)]%
        {loshchilov2019adamw}
\bibfield{author}{\bibinfo{person}{Ilya Loshchilov} {and}
  \bibinfo{person}{Frank Hutter}.} \bibinfo{year}{2019}\natexlab{}.
\newblock \showarticletitle{Decoupled Weight Decay Regularization}. In
  \bibinfo{booktitle}{\emph{The Seventh International Conference on Learning
  Representations (ICLR 2019)}}.
\newblock


\bibitem[Mihaylov et~al\mbox{.}(2018)]%
        {mihaylov2018obqa}
\bibfield{author}{\bibinfo{person}{Todor Mihaylov}, \bibinfo{person}{Peter
  Clark}, \bibinfo{person}{Tushar Khot}, {and} \bibinfo{person}{Ashish
  Sabharwal}.} \bibinfo{year}{2018}\natexlab{}.
\newblock \showarticletitle{Can a Suit of Armor Conduct Electricity? {A} New
  Dataset for Open Book Question Answering}. In
  \bibinfo{booktitle}{\emph{Proceedings of the 2018 Conference on Empirical
  Methods in Natural Language Processing (EMNLP 2018)}}.
  \bibinfo{pages}{2381--2391}.
\newblock
\href{https://doi.org/10.18653/v1/D18-1260}{doi:\nolinkurl{10.18653/v1/D18-1260}}


\bibitem[Mohtashami and Jaggi(2023)]%
        {mohtashami2023landmark}
\bibfield{author}{\bibinfo{person}{Amirkeivan Mohtashami} {and}
  \bibinfo{person}{Martin Jaggi}.} \bibinfo{year}{2023}\natexlab{}.
\newblock \showarticletitle{Random-Access Infinite Context Length for
  Transformers}. In \bibinfo{booktitle}{\emph{Advances in Neural Information
  Processing Systems 36 (NeurIPS 2023)}}.
\newblock


\bibitem[Movahedi et~al\mbox{.}(2026)]%
        {movahedi2026selrope}
\bibfield{author}{\bibinfo{person}{Sajad Movahedi}, \bibinfo{person}{Timur
  Carstensen}, \bibinfo{person}{Arshia Afzal}, \bibinfo{person}{Frank Hutter},
  \bibinfo{person}{Antonio Orvieto}, {and} \bibinfo{person}{Volkan Cevher}.}
  \bibinfo{year}{2026}\natexlab{}.
\newblock \showarticletitle{Selective Rotary Position Embedding}. In
  \bibinfo{booktitle}{\emph{The Fourteenth International Conference on Learning
  Representations (ICLR 2026)}}.
\newblock
\shownote{arXiv:2511.17388}.
\newblock


\bibitem[Paperno et~al\mbox{.}(2016)]%
        {paperno2016lambada}
\bibfield{author}{\bibinfo{person}{Denis Paperno}, \bibinfo{person}{Germ{\'a}n
  Kruszewski}, \bibinfo{person}{Angeliki Lazaridou}, \bibinfo{person}{Ngoc~Quan
  Pham}, \bibinfo{person}{Raffaella Bernardi}, \bibinfo{person}{Sandro
  Pezzelle}, \bibinfo{person}{Marco Baroni}, \bibinfo{person}{Gemma Boleda},
  {and} \bibinfo{person}{Raquel Fern{\'a}ndez}.}
  \bibinfo{year}{2016}\natexlab{}.
\newblock \showarticletitle{The {LAMBADA} Dataset: Word Prediction Requiring a
  Broad Discourse Context}. In \bibinfo{booktitle}{\emph{Proceedings of the
  54th Annual Meeting of the Association for Computational Linguistics (ACL
  2016)}}. \bibinfo{pages}{1525--1534}.
\newblock
\href{https://doi.org/10.18653/v1/P16-1144}{doi:\nolinkurl{10.18653/v1/P16-1144}}


\bibitem[Penedo et~al\mbox{.}(2024)]%
        {penedo2024fineweb}
\bibfield{author}{\bibinfo{person}{Guilherme Penedo}, \bibinfo{person}{Hynek
  Kydl{\'\i}{\v{c}}ek}, \bibinfo{person}{Loubna~Ben Allal},
  \bibinfo{person}{Anton Lozhkov}, \bibinfo{person}{Margaret Mitchell},
  \bibinfo{person}{Colin Raffel}, \bibinfo{person}{Leandro~Von Werra}, {and}
  \bibinfo{person}{Thomas Wolf}.} \bibinfo{year}{2024}\natexlab{}.
\newblock \showarticletitle{The {FineWeb} Datasets: Decanting the Web for the
  Finest Text Data at Scale}. In \bibinfo{booktitle}{\emph{Advances in Neural
  Information Processing Systems 37 (NeurIPS 2024), Datasets and Benchmarks
  Track}}.
\newblock


\bibitem[Peng et~al\mbox{.}(2024)]%
        {peng2024yarn}
\bibfield{author}{\bibinfo{person}{Bowen Peng}, \bibinfo{person}{Jeffrey
  Quesnelle}, \bibinfo{person}{Honglu Fan}, {and} \bibinfo{person}{Enrico
  Shippole}.} \bibinfo{year}{2024}\natexlab{}.
\newblock \showarticletitle{{YaRN}: Efficient Context Window Extension of Large
  Language Models}. In \bibinfo{booktitle}{\emph{The Twelfth International
  Conference on Learning Representations (ICLR 2024)}}.
\newblock


\bibitem[Press et~al\mbox{.}(2022)]%
        {press2022alibi}
\bibfield{author}{\bibinfo{person}{Ofir Press}, \bibinfo{person}{Noah~A.
  Smith}, {and} \bibinfo{person}{Mike Lewis}.} \bibinfo{year}{2022}\natexlab{}.
\newblock \showarticletitle{Train Short, Test Long: Attention with Linear
  Biases Enables Input Length Extrapolation}. In \bibinfo{booktitle}{\emph{The
  Tenth International Conference on Learning Representations (ICLR 2022)}}.
\newblock


\bibitem[Ren et~al\mbox{.}(2025)]%
        {ren2025samba}
\bibfield{author}{\bibinfo{person}{Liliang Ren}, \bibinfo{person}{Yang Liu},
  \bibinfo{person}{Yadong Lu}, \bibinfo{person}{Yelong Shen},
  \bibinfo{person}{Chen Liang}, {and} \bibinfo{person}{Weizhu Chen}.}
  \bibinfo{year}{2025}\natexlab{}.
\newblock \showarticletitle{Samba: Simple Hybrid State Space Models for
  Efficient Unlimited Context Language Modeling}. In
  \bibinfo{booktitle}{\emph{The Thirteenth International Conference on Learning
  Representations (ICLR 2025)}}.
\newblock


\bibitem[Sakaguchi et~al\mbox{.}(2020)]%
        {sakaguchi2020winogrande}
\bibfield{author}{\bibinfo{person}{Keisuke Sakaguchi},
  \bibinfo{person}{Ronan~Le Bras}, \bibinfo{person}{Chandra Bhagavatula}, {and}
  \bibinfo{person}{Yejin Choi}.} \bibinfo{year}{2020}\natexlab{}.
\newblock \showarticletitle{{WinoGrande}: An Adversarial {Winograd} Schema
  Challenge at Scale}. In \bibinfo{booktitle}{\emph{Proceedings of the
  Thirty-Fourth AAAI Conference on Artificial Intelligence (AAAI 2020)}}.
  \bibinfo{pages}{8732--8740}.
\newblock


\bibitem[Sap et~al\mbox{.}(2019)]%
        {sap2019siqa}
\bibfield{author}{\bibinfo{person}{Maarten Sap}, \bibinfo{person}{Hannah
  Rashkin}, \bibinfo{person}{Derek Chen}, \bibinfo{person}{Ronan~Le Bras},
  {and} \bibinfo{person}{Yejin Choi}.} \bibinfo{year}{2019}\natexlab{}.
\newblock \showarticletitle{Social {IQa}: Commonsense Reasoning about Social
  Interactions}. In \bibinfo{booktitle}{\emph{Proceedings of the 2019
  Conference on Empirical Methods in Natural Language Processing and the 9th
  International Joint Conference on Natural Language Processing (EMNLP-IJCNLP
  2019)}}. \bibinfo{pages}{4463--4473}.
\newblock
\href{https://doi.org/10.18653/v1/D19-1454}{doi:\nolinkurl{10.18653/v1/D19-1454}}


\bibitem[Shazeer(2020)]%
        {shazeer2020glu}
\bibfield{author}{\bibinfo{person}{Noam Shazeer}.}
  \bibinfo{year}{2020}\natexlab{}.
\newblock \showarticletitle{{GLU} Variants Improve Transformer}.
\newblock \bibinfo{journal}{\emph{arXiv preprint arXiv:2002.05202}}
  (\bibinfo{year}{2020}).
\newblock


\bibitem[Su et~al\mbox{.}(2024)]%
        {su2024roformer}
\bibfield{author}{\bibinfo{person}{Jianlin Su}, \bibinfo{person}{Murtadha
  Ahmed}, \bibinfo{person}{Yu Lu}, \bibinfo{person}{Shengfeng Pan},
  \bibinfo{person}{Wen Bo}, {and} \bibinfo{person}{Yunfeng Liu}.}
  \bibinfo{year}{2024}\natexlab{}.
\newblock \showarticletitle{{RoFormer}: Enhanced Transformer with Rotary
  Position Embedding}.
\newblock \bibinfo{journal}{\emph{Neurocomputing}}  \bibinfo{volume}{568}
  (\bibinfo{year}{2024}), \bibinfo{pages}{127063}.
\newblock
\href{https://doi.org/10.1016/j.neucom.2023.127063}{doi:\nolinkurl{10.1016/j.neucom.2023.127063}}


\bibitem[Tillet et~al\mbox{.}(2019)]%
        {tillet2019triton}
\bibfield{author}{\bibinfo{person}{Philippe Tillet}, \bibinfo{person}{H.~T.
  Kung}, {and} \bibinfo{person}{David Cox}.} \bibinfo{year}{2019}\natexlab{}.
\newblock \showarticletitle{Triton: An Intermediate Language and Compiler for
  Tiled Neural Network Computations}. In \bibinfo{booktitle}{\emph{Proceedings
  of the 3rd ACM SIGPLAN International Workshop on Machine Learning and
  Programming Languages (MAPL 2019)}}. \bibinfo{publisher}{ACM},
  \bibinfo{pages}{10--19}.
\newblock
\href{https://doi.org/10.1145/3315508.3329973}{doi:\nolinkurl{10.1145/3315508.3329973}}


\bibitem[Vaswani et~al\mbox{.}(2017)]%
        {vaswani2017attention}
\bibfield{author}{\bibinfo{person}{Ashish Vaswani}, \bibinfo{person}{Noam
  Shazeer}, \bibinfo{person}{Niki Parmar}, \bibinfo{person}{Jakob Uszkoreit},
  \bibinfo{person}{Llion Jones}, \bibinfo{person}{Aidan~N. Gomez},
  \bibinfo{person}{{\L}ukasz Kaiser}, {and} \bibinfo{person}{Illia
  Polosukhin}.} \bibinfo{year}{2017}\natexlab{}.
\newblock \showarticletitle{Attention Is All You Need}. In
  \bibinfo{booktitle}{\emph{Advances in Neural Information Processing Systems
  30 (NeurIPS 2017)}}. \bibinfo{pages}{5998--6008}.
\newblock


\bibitem[Yang et~al\mbox{.}(2025)]%
        {yang2025path}
\bibfield{author}{\bibinfo{person}{Songlin Yang}, \bibinfo{person}{Yikang
  Shen}, \bibinfo{person}{Kaiyue Wen}, \bibinfo{person}{Shawn Tan},
  \bibinfo{person}{Mayank Mishra}, \bibinfo{person}{Liliang Ren},
  \bibinfo{person}{Rameswar Panda}, {and} \bibinfo{person}{Yoon Kim}.}
  \bibinfo{year}{2025}\natexlab{}.
\newblock \showarticletitle{{PaTH} Attention: Position Encoding via
  Accumulating Householder Transformations}. In
  \bibinfo{booktitle}{\emph{Advances in Neural Information Processing Systems
  38 (NeurIPS 2025)}}. \bibinfo{pages}{62220--62247}.
\newblock
\href{https://doi.org/10.52202/085713-2080}{doi:\nolinkurl{10.52202/085713-2080}}


\bibitem[Zellers et~al\mbox{.}(2019)]%
        {zellers2019hellaswag}
\bibfield{author}{\bibinfo{person}{Rowan Zellers}, \bibinfo{person}{Ari
  Holtzman}, \bibinfo{person}{Yonatan Bisk}, \bibinfo{person}{Ali Farhadi},
  {and} \bibinfo{person}{Yejin Choi}.} \bibinfo{year}{2019}\natexlab{}.
\newblock \showarticletitle{{HellaSwag}: Can a Machine Really Finish Your
  Sentence?}. In \bibinfo{booktitle}{\emph{Proceedings of the 57th Annual
  Meeting of the Association for Computational Linguistics (ACL 2019)}}.
  \bibinfo{pages}{4791--4800}.
\newblock
\href{https://doi.org/10.18653/v1/P19-1472}{doi:\nolinkurl{10.18653/v1/P19-1472}}


\bibitem[Zhang and Sennrich(2019)]%
        {zhang2019rmsnorm}
\bibfield{author}{\bibinfo{person}{Biao Zhang} {and} \bibinfo{person}{Rico
  Sennrich}.} \bibinfo{year}{2019}\natexlab{}.
\newblock \showarticletitle{Root Mean Square Layer Normalization}. In
  \bibinfo{booktitle}{\emph{Advances in Neural Information Processing Systems
  32 (NeurIPS 2019)}}.
\newblock


\end{thebibliography}

\end{document}